\documentclass{article}

\usepackage{microtype}
\usepackage{graphicx}
\usepackage{subcaption}
\usepackage{booktabs} 

\usepackage{hyperref}

\usepackage[preprint]{icml2026}

\usepackage{amsmath}
\usepackage{amssymb}
\usepackage{mathtools}
\usepackage{amsthm}
\usepackage{multirow} 

\usepackage{amsfonts}
\usepackage{algorithm, algorithmic}
\usepackage{bm}
\usepackage{tabularx}

\newcommand{\bW}{\mathbf{W}}
\newcommand{\bX}{\mathbf{X}}

\newcommand{\calL}{\mathcal{L}}

\newcommand{\calE}{\mathcal{E}}
\newcommand{\prox}{\text{prox}}

\usepackage[capitalize,noabbrev]{cleveref}

\theoremstyle{plain}
\newtheorem{theorem}{Theorem}[section]

\theoremstyle{definition}
\newtheorem{definition}[theorem]{Definition}
\newtheorem{assumption}[theorem]{Assumption}
\theoremstyle{remark}
\newtheorem{remark}[theorem]{Remark}

\usepackage[textsize=tiny]{todonotes}

\icmltitlerunning{Astro: Activation-guided Structured Regularization for Outlier-Robust LLM Post-Training Quantization}

\begin{document}

\twocolumn[
  \icmltitle{Astro: Activation-guided Structured Regularization for \\ Outlier-Robust LLM Post-Training Quantization}



  \icmlsetsymbol{equal}{*}

  \begin{icmlauthorlist}
    \icmlauthor{Xi Chen}{bit}
    \icmlauthor{Ming Li}{bit}
    \icmlauthor{Junxi Li}{bit}
    \icmlauthor{Changsheng Li}{bit}
    \icmlauthor{Peisong Wang}{comp}
    \icmlauthor{Lizhong Ding}{bit}
    \icmlauthor{Ye Yuan}{bit}
    \icmlauthor{Guoren Wang}{bit,hebei}
  \end{icmlauthorlist}

  \icmlaffiliation{bit}{Beijing Institute of Technology}
  \icmlaffiliation{comp}{Institute of Automation, Chinese Academy of Sciences}
  \icmlaffiliation{hebei}{Hebei Province Key Laboratory of Big Data Science and Intelligent Technology}

  \icmlcorrespondingauthor{Changsheng Li}{lcs@bit.edu.cn}


  \vskip 0.3in
]



\printAffiliationsAndNotice{}  

\begin{abstract}
Weight-only post-training quantization (PTQ) is crucial for efficient Large Language Model (LLM) deployment but suffers from accuracy degradation caused by weight and activation outliers. Existing mitigation strategies often face critical limitations: they either yield insufficient outlier suppression or incur significant deployment inefficiencies, such as inference latency, heavy preprocessing, or reliance on complex operator fusion. To resolve these limitations, we leverage a key insight: over-parameterized LLMs often converge to \textit{Flat Minima}, implying a vast equivalent solution space where weights can be adjusted without compromising accuracy. Building on this, we propose \textbf{Astro}, an \underline{\textbf{A}}ctivation-guided \underline{\textbf{St}}ructured \underline{\textbf{R}}egularizati\underline{\textbf{o}}n framework designed to suppress the negative effects of outliers in a hardware-friendly and efficient manner. Leveraging the activation-guided regularization objective, Astro actively reconstructs intrinsically robust weights, aggressively suppressing weight outliers corresponding to high-magnitude activations without sacrificing model accuracy. Crucially, Astro introduces zero inference latency and is orthogonal to mainstream quantization methods like GPTQ. Extensive experiments show that Astro achieves highly competitive performance; notably, on LLaMA-2-7B, it achieves better performance than complex learning-based rotation methods with almost $1/3$ of the quantization time.
\end{abstract}

\section{Introduction}

Large Language Models (LLMs) have achieved remarkable success across diverse natural language processing (NLP) tasks, driven largely by the exponential growth in parameter scale. However, this massive scale introduces significant memory footprints and computational overheads, severely impeding practical deployment. 

To address these resource constraints, weight-only post-training quantization (PTQ) has emerged as a prominent compression strategy, precisely addressing the memory-bound bottleneck of LLM decoding \citep{lin2024awq, frantar2023optq, huang2024billm, kim2024squeezellm}. In the autoregressive generation phase, the primary latency bottleneck stems from the limited memory bandwidth required to stream weights from high bandwidth memory (HBM) to on-chip computing units (SRAM). By compressing weights into low-bit integer formats, weight-only PTQ significantly reduces memory footprints and mitigates data transfer overhead, thereby achieving substantial acceleration in token generation.

However, as illustrated in Figure \ref{fig:outliers}, existing weight-only PTQ methods face a core challenge: the significant outliers in both LLM weights and activations \cite{nipsllmint8, xiao2023smoothquant}. This issue manifests in two mechanisms. First, weight outliers drastically stretch the quantization range, resulting in a coarser quantization granularity that degrades the representation precision for the vast majority of normal values. Second, although weight-only PTQ does not directly quantize activations, activation outliers act as amplifiers during matrix multiplication. Consequently, these outliers severely amplify the quantization error of corresponding weights, ultimately leading to catastrophic collapse in model performance.

Extensive research has been dedicated to mitigating the performance degradation induced by outliers. One research trajectory focuses on differential treatment of outliers \cite{nipsllmint8, kim2024squeezellm, lee2024owq, zhao2024atom}. This includes techniques such as separating outliers from the main weight matrix to store them in high-precision formats, or employing non-uniform quantization algorithms to better fit the heavy-tailed distribution of weights. While effective, these methods often introduce irregular memory access patterns and additional computational branches, which disrupt hardware parallelism and increase inference latency.

Another research trajectory focuses on equivalent mathematical transformations, such as channel-wise scaling or rotation, aimed at smoothing weight or activation distributions to reduce quantization difficulty \cite{lin2024awq, shao2024omniquant, liu2025spinquant, ashkboos2024quarot}. 
However, existing solutions often suffer from limitations such as insufficient outlier suppression (leading to suboptimal accuracy), high preprocessing overhead, or architectural modifications (leading to additional inference latency or requiring complex operator fusion), complicating these methods' actual deployment. Detailed analyses are presented in Appendix \ref{app:inference_overhead}. This context motivates a pivotal question:

\textbf{Can we design a lightweight, hardware-friendly mechanism that sufficiently suppresses the negative impact of outliers?}

We answer this affirmatively by leveraging a critical insight often overlooked in PTQ: as shown in Figure \ref{fig:astro_overview}, over-parameterized LLMs often converge to \textit{Flat Minima} in the loss landscape, not sharp isolated optima \cite{garipov2018loss, draxler2018essentially, maly2023simple}. As detailed in our theoretical \ref{thm:flat_minima}, this implies the existence of a vast equivalent solution space near the original pre-trained weights. Within this solution space, we can search for new weights that yield nearly identical loss values to the original but are numerically robust to quantization.
However, existing methods often fail to exploit these degrees of freedom.

\begin{figure}[t]
    \centering
    \includegraphics[width=1.0\linewidth]{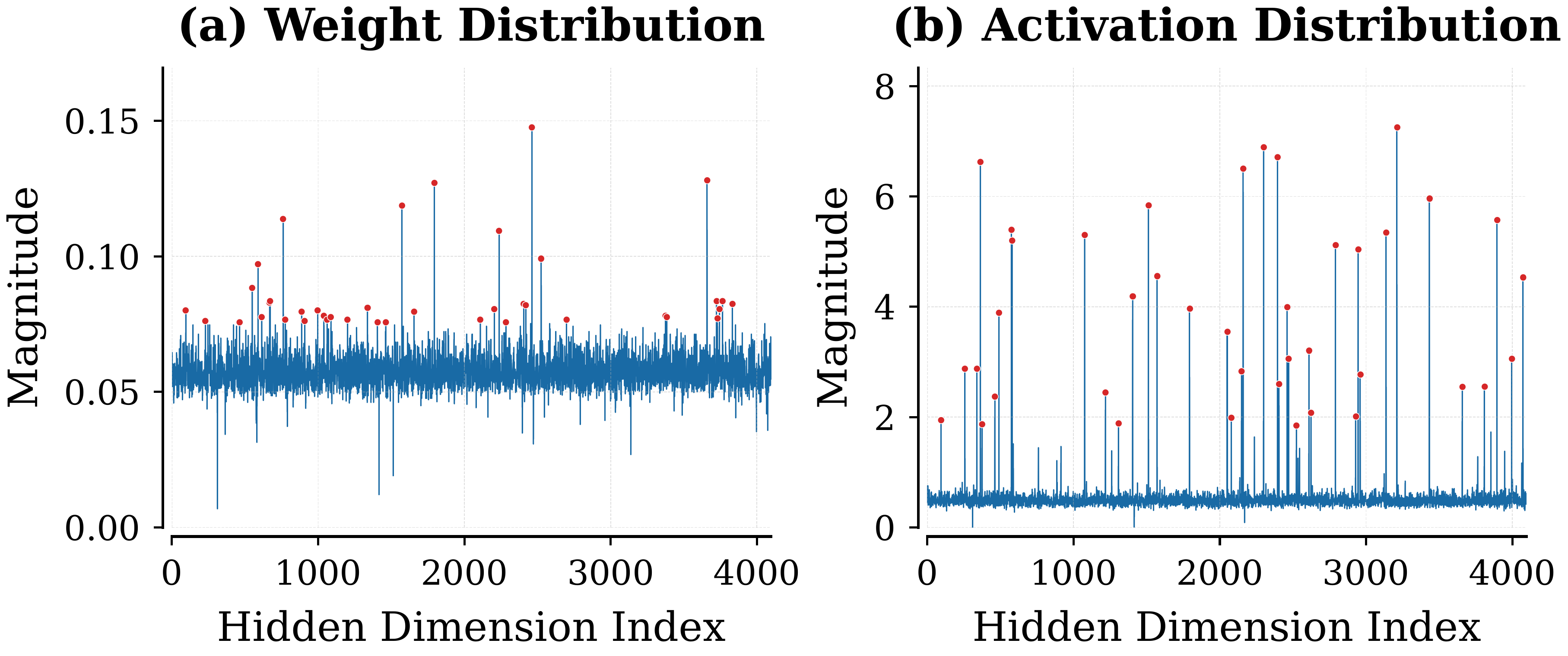} 
    \vspace{-15pt}
    \caption{Magnitude distributions of weight and activations at one layer in LLaMA-2-7B. Both LLM weights and activations exhibit significant outliers (highlighted in red), which cause severe performance degradation in low-bit quantization settings.}
    \label{fig:outliers}
    \vspace{-10pt}
\end{figure}

\begin{figure*}[htbp]
  \centering
  \vspace{-5pt}
  \includegraphics[width=0.7\linewidth]{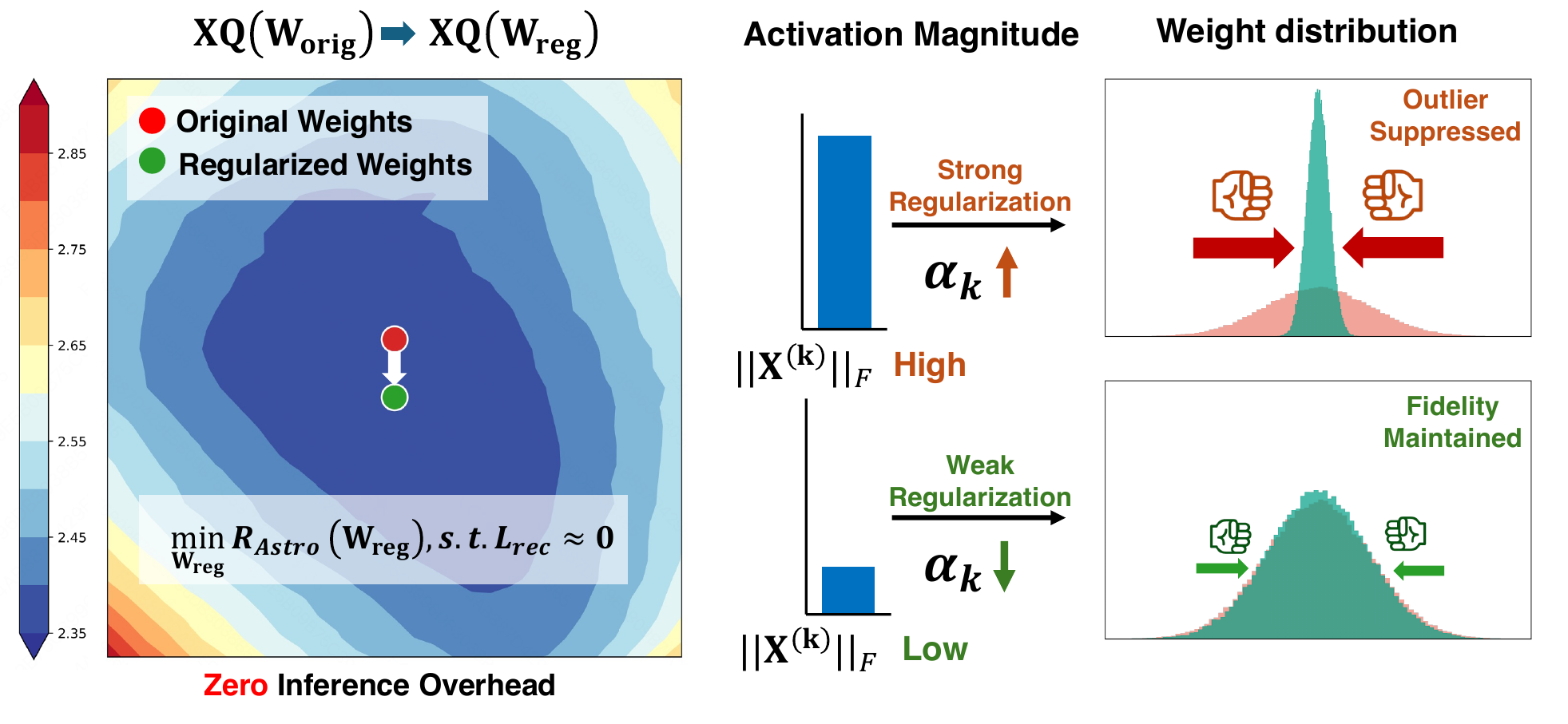}
  \vspace{-5pt}
  \caption{\textbf{Overview of the Astro Framework.} 
  \textbf{Left:} Visualization of the loss landscape for LLaMA-2-7B on WikiText-2. Astro exploits the \textit{Flat Minima} landscape of pre-trained LLMs, searching for quantization-robust weights (Green) that are functionally equivalent to the original (Red). \textbf{Right:} The search is driven by \textit{Activation-guided Structured Regularization}. By coupling regularization strength $\alpha_k$ with activation magnitude $\|\mathbf{X}^{(k)}\|_F$, Astro adaptively suppresses critical weight outliers in high-activation groups while preserving reconstruction fidelity elsewhere. The reconstructed weights directly replace the original, introducing \textbf{zero inference latency}.}
  \label{fig:astro_overview}
  \vspace{-15pt}
\end{figure*}

Guided by this insight, we propose \textbf{Astro} (\underline{\textbf{A}}ctivation-guided \underline{\textbf{St}}ructured \underline{\textbf{R}}egularizati\underline{\textbf{o}}n), a novel plug-and-play framework for outlier-robust LLM quantization. We directly exploit the degrees of freedom provided by these flat minima to actively \textit{reconstruct} the weights toward a configuration that is intrinsically robust to outlier quantization while maintaining model accuracy. 
Guided by our theoretical error bound analysis in Theorem \ref{thm:bound}, we propose an activation-guided structured regularization objective: for critical groups dominated by activation outliers, Astro imposes strong regularization to aggressively suppress weight outliers, thereby blocking the amplification of quantization error. Conversely, in the stable groups, constraints are relaxed to allow weights to fluctuate flexibly within the solution space, maximizing reconstruction fidelity.

Unlike methods that require high-precision decomposition or equivalent transformations, Astro reconstructs and replaces the original weights with robust weights. This ensures compatibility with standard kernels and introduces absolutely no latency overhead during inference. Moreover, Astro can be seamlessly integrated with standard methods (e.g., GPTQ) to further unlock compression potential. 

\begin{itemize}
    \item \textbf{Paradigm Shift via Flat Minima:} We theoretically substantiate that LLMs converge to \textit{Flat Minima}, implying a continuous solution space equivalent to the pre-trained weights (Theorem~\ref{thm:flat_minima}). Leveraging this, we propose a new paradigm that exploits the degrees of freedom within this solution space to \textit{actively reconstruct} weights. This allows us to identify weights that are numerically robust to quantization while strictly preserving original accuracy.
    \item \textbf{Theoretically Grounded Algorithm:}  We derive a quantization error upper bound that reveals a multiplicative coupling between weight outliers and activation magnitudes (Theorem~\ref{thm:bound}). Guided by this, we introduce \textbf{Astro}, an activation-guided structured regularization framework that effectively suppresses weight outliers in high-magnitude activation groups, effectively alleviating the quantization collapse caused by outliers.
    \item \textbf{Zero Inference Overhead and Efficiency:} Extensive experiments demonstrate that Astro achieves highly competitive performance across various low-bit settings. Crucially, Astro guarantees zero inference overhead and hardware friendliness, as the reconstructed weights directly replace the originals. Furthermore, it serves as a highly efficient plug-and-play enhancer for standard methods (e.g., GPTQ), achieving superior results while requiring approximately $1/3$ of the quantization time compared to complex rotation-based methods (Figure \ref{fig:efficiency_comparison}).
\end{itemize}

\section{Related Work}
\subsection{Post-Training Quantization for LLMs}
As Large Language Models (LLMs) scale, post-training quantization (PTQ) has become a widely adopted compression paradigm.
While early round-to-nearest (RTN) methods are simple, they often degrade accuracy.
To address this, optimization-based approaches were introduced: AdaRound \cite{nagel2020adaround} pioneered adaptive rounding optimization to minimize local errors.
Building on this, BRECQ \cite{li2021brecq} advanced the field with block-wise reconstruction, while GPTQ \cite{frantar2023optq} further leveraged second-order Hessian information to enable efficient layer-wise weight updates. Despite their efficacy, these methods overlook the prevalence of extreme outliers in both weights and activations of LLMs, leading to severe performance degradation in low-bit regimes. In contrast, Astro explicitly suppresses outliers by leveraging the flexibility of flat minima to actively reconstruct robust weights without compromising model accuracy. 

\subsection{Outlier Suppression Techniques}
The primary bottleneck in low-bit LLM quantization is the presence of extreme outliers in both weights and activations.
Existing mitigation strategies generally fall into two categories: differential treatment and equivalent mathematical transformation. 

Differential treatment methods like LLM.int8() \cite{nipsllmint8} and SpQR \cite{dettmers2024spqr} decompose matrices to isolate outliers, storing them in high precision (e.g., FP16) while quantizing the rest.
While accurate, this approach creates irregular memory access patterns and requires conditional branch execution, which significantly hampers hardware parallelism and increases inference latency. Equivalent mathematical transformation methods are applied to smooth distributions of weight or activation for easier quantization.
AWQ \cite{lin2024awq} and OmniQuant \cite{shao2024omniquant} leverage channel-wise scaling to suppress the impact of outliers, migrating the quantization difficulty. QuIP \cite{chee2023quip} and QuaRot \cite{ashkboos2024quarot} introduce rotation transformations to disperse outliers across all channels. SpinQuant \cite{liu2025spinquant} further advances this approach by optimizing these rotation matrices via learnable optimization.

However, these methods face significant limitations. First, rotation-based methods modify the model inference architecture, resulting in either additional inference latency or the need for complex operator fusion. Detailed analyses are presented in Appendix \ref{app:inference_overhead}. Second, as illustrated in Figure \ref{fig:efficiency_comparison}, methods like OmniQuant and SpinQuant often require expensive optimization processes to determine transformation matrices. In contrast, Astro leverages lightweight structured regularization to directly replace original pre-trained weights with quantization-robust alternatives, ensuring zero additional inference latency.

\subsection{Regularization-based Quantization}
Regularization techniques, traditionally used to constrain model complexity, have recently been explored for improving quantization robustness. \citet{kundu2023r2loss} incorporates regularization into the training objective to constrain weight ranges. However, this approach falls within the scope of quantization-aware training (QAT), rendering it impractical for LLMs due to high training costs. \citet{maly2023simple} exploits to find equivalent and more quantization-friendly weights via a deterministic pre-processing step. However, this method is restricted to small-scale architectures (e.g., ResNet) and failed to address the distinct outlier characteristics of transformer-based LLMs. 
Recently, MagR \cite{zhang2024magr} extends such regularization to LLM PTQ by leveraging the empirical rank deficiency of activation matrices. 
However, MagR relies on a \textit{heuristic} regularization objective that lacks rigorous theoretical guidance. 
It treats weights as isolated entities, applying uniform regularization to all weights, ignoring the critical fact that quantization error is multiplicatively coupled with activation magnitude.
Consequently, it fails to effectively suppress significant weight outliers corresponding to high-magnitude activation, leading to suboptimal quantization performance. 
In contrast, Astro is guided by our derived theoretical error bound. 
By exploiting the degrees of freedom provided by \textit{Flat Minima}, whose existence is theoretically substantiated in Theorem~\ref{thm:flat_minima}, Astro explicitly incorporates the activation magnitude into the regularization, thereby achieving a theoretically grounded equilibrium between aggressive outlier suppression and weight reconstruction fidelity.

\begin{figure}[t]
  \centering
  \vspace{-5pt}
  \includegraphics[width=0.75\linewidth]{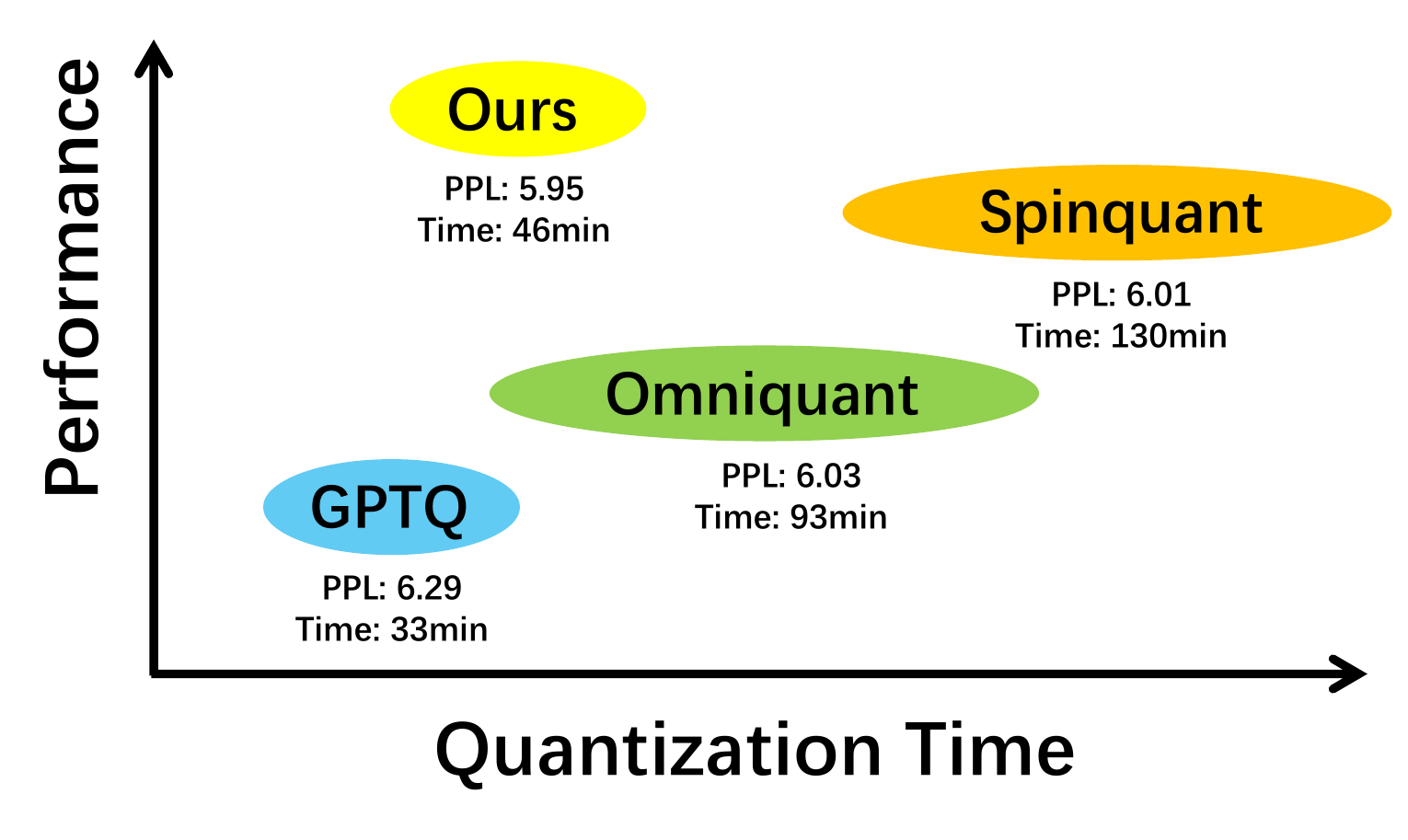}
   \vspace{-7pt}
  \caption{Efficiency vs. Performance Trade-off (LLaMA-2-7B, W3A16g128). Comparison of perplexity (PPL $\downarrow$) and total quantization time. ``W3A16g128'' denotes 3-bit weights, 16-bit activations, and a group size of 128.}
  \vspace{-15pt}
  \label{fig:efficiency_comparison}
\end{figure}

\section{Preliminaries}
\label{sec:preliminaries}

In this section, we formally define the problem of weight-only post-training quantization and analyze the limitations of existing approaches in handling outliers in LLMs.

\subsection{Problem Formulation}
\label{subsec:problem_formulation}

Consider a standard linear layer $\mathbf{Y} = \mathbf{X}\mathbf{W}$, where $\mathbf{X} \in \mathbb{R}^{N \times C_{in}}$ represents the input activation matrix and $\mathbf{W} \in \mathbb{R}^{C_{in} \times C_{out}}$ denotes the pre-trained floating-point weight matrix. Here, $N$, $C_{in}$, and $C_{out}$ denote the number of tokens, input channels, and output channels, respectively.

To enhance inference efficiency, weight-only PTQ typically adopts \textit{symmetric uniform quantization}. The quantization function $\mathcal{Q}(\cdot)$ is defined as:
\begin{equation}
\mathcal{Q}(\mathbf{W}) = \mathbf{S} \cdot \text{clamp}\left( \left\lfloor \frac{\mathbf{W}}{\mathbf{S}} \right\rceil, -2^{b-1}, 2^{b-1}-1 \right),
\end{equation}
where $\mathcal{Q}(\mathbf{W})$ denotes the dequantized weights, $b$ represents the bit-width, $\lfloor \cdot \rceil$ indicates the round-to-nearest-integer operation, and $\mathbf{S}$ denotes the matrix of scaling factors. To mitigate quantization errors, modern LLMs widely adopt \textit{fine-grained group-wise quantization} \cite{frantar2023optq, lin2024awq}. This strategy significantly improves quantization precision while maintaining hardware efficiency and introducing negligible storage overhead. Specifically, for every output channel, we partition the input dimension $C_{in}$ into contiguous blocks of size $g$. Let $\mathbf{w} \in \mathbb{R}^g$ denote the weight vector within such a specific group. The corresponding scalar scaling factor $s$ is computed as:
\begin{equation} 
s = \frac{\max(|\mathbf{w}|)}{2^{b-1}-1}. 
\end{equation}
Given a small calibration dataset $\mathcal{D}_{calib}$, the primary objective of LLM PTQ is to identify a quantized weight $\mathcal{Q}(\mathbf{W})$ that approximates the full-precision model's output. This is typically modeled as minimizing the layer-wise reconstruction error \cite{frantar2023optq, nagel2020adaround}:
\begin{equation}
\min_{\mathcal{Q}(\mathbf{W})} \mathbb{E}_{\mathbf{X} \sim \mathcal{D}_{calib}} \left[ \| \mathbf{X}\mathbf{W} - \mathbf{X}\mathcal{Q}(\mathbf{W}) \|_F^2 \right].
\end{equation}

\subsection{Challenges of Outliers and Existing Solutions}

As illustrated in Figure \ref{fig:outliers}, both weights and activations of LLMs exhibit substantial outliers, thereby hindering model performance under low-bit quantization. To address this, current mainstream approaches attempt to mitigate this via equivalent mathematical transformations. These methods can be unified as searching for a linear operator $\mathbf{T}$:
\begin{equation}
\mathbf{Y} = \mathbf{X}\mathbf{W} = (\mathbf{X}\mathbf{T}^{-1})(\mathbf{T}\mathbf{W}) = \mathbf{X}'\mathbf{W}'.
\end{equation}
Here, $\mathbf{T}$ is designed to obtain smoother distributions $\mathbf{W}'$ or $\mathbf{X}'$. Common instances include channel-wise scaling (diagonal $\mathbf{T}$) and rotation (orthogonal $\mathbf{T}$).

However, these methods often suffer from the following limitations. 
(i) While the weight transformation $\mathbf{T}\mathbf{W}$ can be fused offline, the activation transformation $\mathbf{X}\mathbf{T}^{-1}$ may increase inference latency for online computation or require additional complex kernel fusion implementation. Detailed analysis is presented in Appendix \ref{app:inference_overhead}.
(ii) As illustrated in Figure \ref{fig:efficiency_comparison}, methods such as OmniQuant and SpinQuant require significant preprocessing overhead to search the optimal transformation matrix $\mathbf{T}$. These constraints highlight the urgent need for a lightweight, hardware-friendly mechanism that significantly eliminates the negative impact of outliers.

\section{Proposed Method}
\label{sec:method}
In this section, we introduce \textbf{Astro}, an \underline{\textbf{A}}ctivation-guided \underline{\textbf{St}}ructured \underline{\textbf{R}}egularizati\underline{\textbf{o}}n framework designed to aggressively suppress outliers while ensuring zero inference overhead. 
First, in Section \ref{subsec:feasibility}, we theoretically and experimentally demonstrate that by exploiting the degrees of freedom within \textit{Flat Minima}, we can actively reconstruct weights to achieve outlier suppression with zero inference overhead and model accuracy preserved.
Guided by this paradigm, Section \ref{subsec:coupling_effect} formally derives the coupling effect between weight outliers and activation magnitude, which directly motivates our activation-guided structured regularization objective.
Finally, Section~\ref{subsec:algorithm} details the regularization algorithm.

\subsection{Exploiting Flat Minima for Weight Reconstruction}
\label{subsec:feasibility}

Most current weight-only PTQ methods apply equivalent transformations (e.g., channel-wise scaling or rotation) directly to the pre-trained weights \cite{lin2024awq, shao2024omniquant, liu2025spinquant}. However, they often overlook the intrinsic loss landscape flatness of over-parameterized LLMs. As illustrated in \cref{fig:astro_overview}, the pre-trained weights of LLMs typically converge to a \textit{Flat Minima} rather than a sharp, isolated optima \cite{garipov2018loss, draxler2018essentially, maly2023simple}. This implies that in the vicinity of the original pre-trained weights $\bm{\Theta}_{\text{orig}}$, there exists a vast solution space yielding nearly identical loss values but exhibiting distinct numerical properties. We propose to exploit these degrees of freedom to actively \textit{reconstruct} the weights into a quantization-robust configuration.

Formally, we define $\mathcal{R}_{\text{quant}}$ as a metric of \textit{quantization difficulty} (to be concretized in \cref{subsec:coupling_effect}). Our goal is to solve the following constrained optimization problem:
\begin{equation}
\label{eq:global_objective}
    \min_{\bm{\Theta}} \quad \mathcal{R}_{\text{quant}}(\bm{\Theta}) \quad \text{s.t.} \quad  |\mathcal{L}(\bm{\Theta}) - \mathcal{L}(\bm{\Theta}_{\text{orig}})| \le \epsilon,
\end{equation}
where $\mathcal{L}(\cdot)$ is the global objective function and $\epsilon$ is a tolerance threshold ensuring negligible accuracy degradation. To justify its feasibility, we introduce the following theoretical framework.

\begin{assumption}[$\epsilon$-Rank Deficiency of Global Hessian]
\label{ass:hessian}
Let $\mathbf{H} = \nabla^2 \calL(\bm{\Theta}_{\text{orig}}) \in \mathbb{R}^{D \times D}$ be the global Hessian matrix at convergence, where $D$ denotes the total number of weights. Let $\{\lambda_1, \dots, \lambda_D\}$ be the eigenvalues of $\mathbf{H}$ sorted in descending order. Given the over-parameterization of LLMs, we assume $\mathbf{H}$ is effectively rank-deficient with a numerical rank $r \ll D$. Specifically, for the tail spectrum $i > r$, the curvature is bounded by a negligible scalar $\gamma$:
\begin{equation}
    \forall i > r, \quad |\lambda_i| \le \gamma, \quad \text{where } \gamma \ll \lambda_1.
\end{equation}
\end{assumption}

\begin{definition}[Flat Subspace]
\label{def:flat_subspace}
Based on Assumption~\ref{ass:hessian}, we define the \textit{Flat Subspace} $\mathcal{V}_{\text{flat}} \subset \mathbb{R}^D$ as the span of eigenvectors $\{\mathbf{v}_{r+1}, \dots, \mathbf{v}_D\}$ associated with the trivial eigenvalues. Perturbations restricted to this subspace induce minimal loss variation.
\end{definition}

\begin{theorem}[Existence of Robust Solution Region]
\label{thm:flat_minima}
Consider a pre-trained LLM parameterized by $\bm{\Theta}_{\text{orig}}$. Under Assumption~\ref{ass:hessian}, for any loss tolerance $\epsilon > 0$, there exists a convex region $\mathcal{M}_{\epsilon} \subset \mathcal{V}_{\text{flat}}$ centered at $\bm{\Theta}_{\text{orig}}$, defined as:
\begin{equation}
    \mathcal{M}_{\epsilon} = \left\{ \bm{\Theta} = \bm{\Theta}_{\text{orig}} + \Delta \bm{\Theta} \mid \Delta \bm{\Theta} \in \mathcal{V}_{\text{flat}}, \|\Delta \bm{\Theta}\|_2 \le \delta \right\},
\end{equation}
where the radius is $\delta = \sqrt{2\epsilon/\gamma}$. Within this region, for all $\bm{\Theta} \in \mathcal{M}_{\epsilon}$, the degradation of the quadratic approximation of the loss is bounded by $\epsilon$.
\end{theorem}

\begin{remark}
Theorem~\ref{thm:flat_minima} implies that $\bm{\Theta}_{\text{orig}}$ is not a unique solution but rather one instance on a continuous manifold. Astro aims to transition from $\bm{\Theta}_{\text{orig}}$ to a more quantization-robust candidate $\bm{\Theta} \in \mathcal{M}_{\epsilon}$.
\end{remark}

\textbf{From Global to Layer-wise Regularization.} 
Directly solving \cref{eq:global_objective} is computationally intractable. To derive a feasible algorithm, we adopt the standard assumption in Hessian-based quantization \cite{frantar2023optq, nagel2020adaround, frantar2022obq} that the inter-layer dependencies are negligible. 
This implies approximating the global Hessian $\mathbf{H}$ as a block-diagonal matrix, where each block $\mathbf{H}^{(l)} \approx 2(\mathbf{X}^{(l)})^\top \mathbf{X}^{(l)}$ corresponds to a single layer.
Consequently, the global optimization problem decomposes into independent \textbf{layer-wise} sub-problems:
\begin{equation}
\label{eq:layerwise_proxy}
    \min_{\mathbf{W}^{(l)}} \mathcal{R}_{\text{quant}}(\mathbf{W}^{(l)}) \text{  s.t.  }  \|\mathbf{X}^{(l)}\mathbf{W}^{(l)}-\mathbf{X}^{(l)}\mathbf{W}^{(l)}_{\text{orig}}\|_F \le \epsilon.
\end{equation}

\subsection{Activation-guided Structured Regularization}
\label{subsec:coupling_effect}

Having established the existence of the equivalent solution space $\mathcal{M}_{\epsilon}$, the pivotal question remains: within this solution space, which specific weight configuration minimizes quantization difficulty $\mathcal{R}_{\text{quant}}$? To answer this, we first analyze the upper bound of the quantization error.

\textbf{Theoretical Error Bound.}
Consider a single layer with weights $\bW$ and input activations $\bX$. In the context of group-wise quantization with group size $g$, let $\bX^{(k)}$ denote the activation columns corresponding to the $k$-th input group, and $\bW^{(k, j)}$ denote the weight vector for the $j$-th output channel in that group.
We derive the following upper bound for the layer-wise quantization error:
\begin{theorem}[Activation-Weight Coupled Error Bound]
\label{thm:bound}
Assume a symmetric uniform quantization with bit-width $b$. For any input batch $\bX$, the Frobenius norm of the layer-wise quantization error $\calE = \|\bX\bW - \bX\mathcal{Q}(\bW)\|_F$ is bounded by:
\begin{equation}
\label{eq:error_bound}
    \calE \le \frac{\sqrt{g}}{2 \cdot (2^{b-1}-1)} \sum_{j=1}^{C_{out}} \sum_{k=1}^{K} \|\bX^{(k)}\|_F \cdot \|\bW^{(k, j)}\|_\infty.
\end{equation}
\end{theorem}
\begin{remark}
Theorem~\ref{thm:bound} reveals that the impact of weight outliers is not static; it is multiplicatively coupled with the magnitude of the corresponding activations. A weight outlier in a high-activation group causes significantly more damage to the model output than one in a low-activation channel.
\end{remark}
\textbf{Deriving the Regularization Objective.} Motivated by Theorem~\ref{thm:bound}, we propose to transform the abstract quantization difficulty metric $\mathcal{R}_{\text{quant}}$ in Eq.~\eqref{eq:layerwise_proxy} into a tractable regularization objective. Our goal is to explicitly minimize the derived error bound.
By treating the activation magnitude $\|\bX^{(k)}\|_F$ as a prior constant computed from calibration data, we formulate an activation-guided structured regularization objective:
\begin{equation}
\label{eq:reg_definition}
    \mathcal{R}_{\text{Astro}}(\mathbf{W}) \triangleq \sum_{j=1}^{C_{out}} \sum_{k=1}^{K} \alpha_{k} \|\mathbf{W}^{(k, j)}\|_\infty, 
\end{equation}
where the adaptive coefficient $\alpha_{k}$ is directly proportional to the group-wise activation magnitude:
\begin{equation}
    \alpha_{k} = \frac{\|\bX^{(k)}\|_F}{\frac{1}{K}\sum_{i=1}^{K}\|\bX^{(i)}\|_F}.
\end{equation}
As illustrated in Figure \ref{fig:astro_overview}, this formulation enforces a precise, theoretically grounded trade-off:
for \textbf{critical groups} (high $\alpha_k$), regularizer is forced to aggressively suppress weight outliers (minimize $\|\bW^{(k, j)}\|_\infty$) to reduce the error bound; for \textbf{stable groups} (low $\alpha_k$), the constraints are naturally relaxed, allowing the weights to fluctuate freely within the flat minima to preserve reconstruction fidelity.
Unlike MagR \cite{zhang2024magr} that applies uniform constraints, Astro employs a fine-grained, activation-guided regularization strategy to rigorously minimize the theoretical error bound.

\subsection{The Astro Algorithm Framework}
\label{subsec:algorithm}

Building upon the theoretical insights of the flat minima landscape (Section~\ref{subsec:feasibility}) and the coupling effect of outliers  (Section~\ref{subsec:coupling_effect}), we formally propose \textbf{Astro}. Our goal is to find a weight matrix that minimizes the theoretical coupled quantization error bound while remaining within the flat basin of the original loss.

\subsubsection{Column-wise Formulation}
Since both the reconstruction loss (Frobenius norm) and the structured regularization are additive across output channels, the global optimization problem decomposes into $C_{out}$ independent sub-problems. This allows us to simplify the modeling from matrix space to vector space, enabling massive parallelization.

For a single output channel $j \in \{1, \dots, C_{out}\}$. Let $\mathbf{w} \in \mathbb{R}^{C_{in}}$ denote the corresponding weight column vector (we omit the subscript $j$ for brevity), and $\mathbf{y} = \mathbf{X}\mathbf{w}_{\text{orig}}$ be the target output.
The objective function for this specific channel decouples as:
\begin{equation}
\label{eq:astro_objective_vec}
\min_{\mathbf{w}} \mathcal{J}(\mathbf{w}) = \underbrace{\frac{1}{2}\|\mathbf{X}\mathbf{w} - \mathbf{y}\|_2^2}_{\mathcal{L}_{rec}} + \underbrace{\beta \sum_{k=1}^{K} \alpha_k \|\mathbf{w}^{(k)}\|_\infty}_{\mathcal{R}_{\text{Astro}}},
\end{equation}
where $\mathbf{w}^{(k)} \in \mathbb{R}^g$ represents the sub-vector of weights corresponding to the $k$-th input group, and $\alpha_k$ is the group-specific adaptive coefficient defined in Eq.~\eqref{eq:reg_definition}. 

\subsubsection{Group-wise Proximal Gradient Descent}
\label{subsubsec:pgd}

Since the structured regularization term $\mathcal{R}_{\text{Astro}}$ in Eq.~\eqref{eq:astro_objective_vec} involves non-differentiable $L_\infty$, standard gradient descent is inapplicable. Following \citet{zhang2024magr}, we solve this using proximal gradient descent (PGD) \cite{parikh2014proximal}. For a single channel, the update at iteration $t$ consists of a gradient descent step followed by a structured projection.

\textbf{1. Gradient Step (Fidelity Update).} 
First, we update the weights to minimize the reconstruction error $\mathcal{L}_{rec}$:
\begin{equation}
    \mathbf{v}_t = \mathbf{w}_t - \eta \mathbf{X}^\top (\mathbf{X}\mathbf{w}_t - \mathbf{y}).
\end{equation}

\textbf{2. Proximal Step (Outlier Suppression).} 
Next, we apply the proximal operator to enforce the activation-guided constraints. Critically, the regularization term $\sum \alpha_k \|\mathbf{w}^{(k)}\|_\infty$ is separable across groups. Thus, the proximal mapping applies independently to each group sub-vector $\mathbf{v}_t^{(k)}$:
\begin{equation}
    \mathbf{w}_{t+1}^{(k)} = \prox_{\eta \beta \alpha_k \|\cdot\|_\infty} (\mathbf{v}_t^{(k)}).
\end{equation}
By leveraging the Moreau Decomposition \cite{parikh2014proximal}, this proximal operator has a highly efficient closed-form solution, computed via projection onto the $L_1$-ball dual norm:
\begin{equation}
\label{eq:moreau_proj}
\mathbf{w}_{t+1}^{(k)} = \mathbf{v}_t^{(k)} - \tau_k \cdot \operatorname{proj}_{\mathcal{B}_1} \left( \frac{\mathbf{v}_t^{(k)}}{\tau_k} \right), \quad \text{with} \quad \tau_k = \eta \beta \alpha_k.
\end{equation}
Here, $\operatorname{proj}_{\mathcal{B}_1}$ projects a vector onto the unit $L_1$-ball, which can be solved in $\mathcal{O}(g \log g)$ time (see Algorithm \ref{alg:l1_proj} in Appendix). 
To ensure convergence, we set the step size to the tightest bound $\eta = 1/\lambda_{\max}(\mathbf{X}^\top \mathbf{X})$, where $\lambda_{\max}$ corresponds to the Lipschitz constant of the linear least-squares gradient \cite{zhang2024magr}.

\begin{table*}[t]
\centering
\caption{Weight-only quantization results (PPL) of LLaMA-2 models on WikiText2. Results for FP16, RTN, GPTQ, AWQ, and OmniQuant are taken from  \citet{shao2024omniquant}.}
\vspace{-5pt}
\label{tab:llama2_quant_results}
\small
\setlength{\tabcolsep}{4pt}
\begin{tabular*}{\textwidth}{@{\extracolsep{\fill}}lccccccccc}
\toprule
\multirow{3}{*}{\textbf{Method}} & \multicolumn{3}{c}{\textbf{2-bit / Group-size64}} & \multicolumn{3}{c}{\textbf{3-bit / Group-size128}} & \multicolumn{3}{c}{\textbf{4-bit / Group-size128}} \\
\cmidrule(lr){2-4} \cmidrule(lr){5-7} \cmidrule(lr){8-10}
 & \textbf{7B} & \textbf{13B} & \textbf{70B} & \textbf{7B} & \textbf{13B} & \textbf{70B} & \textbf{7B} & \textbf{13B} & \textbf{70B} \\
\midrule
FP16 & 5.47 & 4.88 & 3.31 & 5.47 & 4.88 & 3.31 & 5.47 & 4.88 & 3.31 \\
\midrule
RTN & 431.97 & 26.22 & 10.31 & 6.66 & 5.51 & 3.97 & 5.72 & 4.98 & 3.46 \\
GPTQ & 20.85 & 22.44 & NAN & 6.29 & 5.42 & 3.85 & 5.61 & 4.98 & 3.42 \\
AWQ & 2.1e5 & 1.2e5 & - & 6.24 & 5.32 & - & 5.62 & 4.97 & - \\
OmniQuant & 9.62 & 7.56 & 6.11 & 6.03 & 5.28 & 3.78 & 5.58 & 4.95 & 3.40 \\
SpinQuant & 21.83 & 10.17 & - & 6.01 & 5.29 & - & \textbf{5.56} & 4.96 & - \\
MagR & 11.11 & 9.66 & 5.73 & 6.00 & 5.23 & 3.71 & 5.61 & 4.96 & 3.43 \\
Astro & \textbf{8.98} & \textbf{7.43} & \textbf{5.20} & \textbf{5.95} & \textbf{5.20} & \textbf{3.69} & \textbf{5.56} & \textbf{4.94} & \textbf{3.38} \\
\bottomrule
\end{tabular*}
\vspace{-10pt}
\end{table*}

\section{Experiments}
\label{sec:experiments}

In this section, we evaluate the effectiveness of Astro across two key dimensions: perplexity and downstream task accuracy. Additionally, we conduct ablation studies to verify the effectiveness of our method. In all reported experiments, both Astro and MagR are implemented as pre-processing modules integrated with GPTQ. Our evaluation primarily focuses on group-wise weight-only quantization, a standard configuration widely adopted in industrial applications.

\subsection{Experimental Settings}
\label{subsec:setup}

\textbf{Models and Datasets.} We conduct extensive evaluations on the LLaMA-2 model family \cite{touvron2023llama2}, spanning the 7B, 13B, and 70B parameter scales. Results for LLaMA-3.1 \cite{grattafiori2024llama3} are also provided in Appendix \ref{app:llama3_results} for further comparison. 

For the calibration process, we randomly sample a training set of 128 sequences from WikiText-2 \cite{merity2016pointer-wiki}, each with a length of 2,048 tokens. To comprehensively assess quantization quality, we employ three categories of benchmarks: 
(1) \textbf{Perplexity (PPL)} evaluated on the WikiText-2 test set; 
(2) \textbf{Zero-shot benchmarks}, including ARC-Challenge, ARC-Easy \cite{clark2018think-arc}, and HellaSwag \cite{zellers2019hellaswag}; 
(3) \textbf{Few-shot (5-shot) benchmarks}, comprising MMLU \cite{hendrycks2021measuring-mmlu}, TriviaQA \cite{joshi2017triviaqa}, and Natural Questions (NQ) \cite{kwiatkowski2019natural}.

\textbf{Baselines.} We compare Astro against a diverse set of mainstream quantization methods to demonstrate its superiority. 
(1) RTN (Round-to-Nearest) is included as the vanilla baseline. 
(2) GPTQ \cite{frantar2023optq} is selected as the representative Hessian-guided layer-wise reconstruction approach. 
(3) Channel-wise scaling methods, including AWQ \cite{lin2024awq} and OmniQuant \cite{shao2024omniquant}. 
(4) Rotation-based method SpinQuant \cite{liu2025spinquant}.
(5) Regularization-based method MagR \cite{zhang2024magr}.

\textbf{Implementation Details.} We employ 200 iterations of proximal gradient descent (PGD) to solve the regularization objective per layer. Regarding the quantization configuration, we adopt a group size of 128 for 3-bit and 4-bit settings, while reducing the group size to 64 for the 2-bit case to mitigate accuracy degradation at extremely low bit-widths. The regularization hyperparameter $\beta$ is tuned to strike a balance between reconstruction fidelity and outlier suppression; a detailed sensitivity analysis is provided in Section \ref{subsec:ablation}. For evaluation, we utilize the \texttt{lm-evaluation-harness} framework~\cite{eval-harness} for all zero-shot and few-shot benchmarks. All experiments are conducted on NVIDIA A100 GPUs.

\subsection{Perplexity Evaluation}
\label{subsec:wikitext_results}

As summarized in Table \ref{tab:llama2_quant_results}, Astro consistently demonstrates superior performance compared to established baselines across various model scales and bit-widths. This advantage is particularly evident in the challenging 2-bit regime, where the extreme scarcity of representation levels renders models highly vulnerable to outliers. Specifically, Astro significantly outperforms AWQ, a representative grid-search-based equivalent transformation method. Even when compared to heavy optimization-based approaches such as OmniQuant and SpinQuant, Astro achieves comparable or even superior accuracy. 

Beyond predictive performance, Astro offers significant advantages in terms of computational efficiency. As illustrated in Figure \ref{fig:efficiency_comparison}, Astro typically reduces quantization time by nearly 50\% or more compared to SpinQuant and OmniQuant, while maintaining higher performance. Furthermore, Astro serves as an effective pre-processing step that mitigates outliers, providing a seamless, plug-and-play enhancement to GPTQ. Additional experiments for enhancing RTN are provided in Appendix \ref{app:ablition_rtn}.
Finally, the substantial performance margin over MagR empirically validates that activation-guided structured regularization is indispensable for preserving the model's representational capacity in the presence of extreme outliers.

\begin{table*}[t]
\centering
\caption{Main results of \textbf{LLaMA-2} on standard benchmarks. The best and second-best results among quantization methods are highlighted in \textbf{bold} and \underline{underline}, respectively.}
\vspace{-3pt}
\label{tab:benchmark_comparison}
\small
\begin{tabular}{ll|ccccc|cc|ccc}
\toprule

\multirow{2}{*}{\textbf{Size}} & \multirow{2}{*}{\textbf{Method}} & 
\multicolumn{5}{c|}{\textbf{MMLU (5-shot)}} & 
\multicolumn{2}{c|}{\textbf{Knowledge (5-shot)}} & 
\multicolumn{3}{c}{\textbf{Common Sense (0-shot)}} \\

\cmidrule(lr){3-7} \cmidrule(lr){8-9} \cmidrule(lr){10-12}

& & \textbf{Avg.} & \textit{STEM} & \textit{Human.} & \textit{Social} & \textit{Other} 
& \textbf{TriviaQA} & \textbf{NQ} 
& \textbf{ARC-C} & \textbf{ARC-E} & \textbf{HellaS.} \\
\midrule

\multirow{4}{*}{\textbf{7B}} 
& fp16        & 45.9 & 37.4 & 43.0 & 51.9 & 53.2 & 64.4 & 26.7 & 45.1 & 73.8 & 76.2 \\
\cmidrule{2-12}
& OmniQuant & \underline{39.8} & \textbf{33.9} & \underline{36.9} & \underline{43.3} & \underline{46.9} & 54.4 & \underline{20.4} & \underline{42.2} & \underline{71.1} & \textbf{73.4} \\
& MagR        & 37.3 & 31.8 & 34.5 & 40.7 & 43.3 & \underline{55.3} & \underline{20.4} & 40.5 & 69.7 & 70.2 \\
& Astro       & \textbf{40.5} & \underline{32.6} & \textbf{38.2} & \textbf{45.0} & \textbf{47.6} & \textbf{55.9} & \textbf{20.5} & \textbf{43.3} & \textbf{71.2} & \underline{72.2} \\
\midrule

\multirow{4}{*}{\textbf{13B}} 
& fp16        & 55.4 & 44.4 & 53.8 & 63.5 & 61.3 & 70.4 & 30.8 & 48.8 & 76.6 & 79.7 \\
\cmidrule{2-12}
& OmniQuant & 51.0 & \underline{42.8} & \underline{47.5} & \underline{58.8} & 56.9 & 64.5 & 25.5 & \underline{46.5} & 73.4 & \textbf{76.9} \\
& MagR        & \underline{51.2} & 42.2 & \textbf{48.5} & 58.3 & \underline{57.3} & \underline{65.2} & \textbf{26.5} & \textbf{47.4} & \textbf{75.6} & 74.2 \\
& Astro       & \textbf{52.2} & \textbf{43.1} & \textbf{48.5} & \textbf{60.4} & \textbf{58.8} & \textbf{66.6} & \textbf{26.5} & 46.2 & \underline{74.7} & \underline{76.2} \\
\midrule

\multirow{4}{*}{\textbf{70B}} 
& fp16        & 69.2 & 58.3 & 65.2 & 80.5 & 75.2 & 80.0 & 38.8 & 57.2 & 79.5 & 84.4 \\
\cmidrule{2-12}
& OmniQuant & \textbf{67.3} & \textbf{55.2} & \underline{64.1} & \textbf{78.5} & \textbf{73.1} & 76.8 & \textbf{36.1} & 54.3 & 78.5 & \textbf{83.1} \\
& MagR        & 66.1 & 53.9 & 63.0 & 77.0 & \underline{72.4} & \textbf{77.5} & 35.7 & \textbf{56.2} & \textbf{79.5} & 81.1 \\
& Astro       & \underline{66.9} & \underline{54.7} & \textbf{64.2} & \underline{77.9} & 72.2 & \textbf{77.5} & \underline{35.8} & \underline{54.8} & \underline{79.3} & \underline{81.7} \\

\bottomrule
\end{tabular}
\vspace{-5pt}
\end{table*}

\subsection{Zero-shot and Few-shot Performance}
\label{subsec:benchmark_results}

Our comparative analysis focuses on the challenging 3-bit quantization regime. We select OmniQuant and MagR as primary baselines, given that these methods demonstrate competitive performance across various bit settings in the perplexity evaluations presented in Table \ref{tab:llama2_quant_results}. 

On smaller model scales, Astro consistently outperforms the baselines across the majority of tasks. In particular, Astro achieves substantial performance gains over MagR on the MMLU benchmark for both 7B and 13B models. These results suggest that activation-guided reconstruction effectively preserves the critical weights required for complex reasoning.
For the LLaMA-2-70B model, Astro delivers a performance margin over MagR across the majority of evaluated benchmarks, providing strong empirical evidence for the superiority of our activation-guided structured regularization at scale. While OmniQuant exhibits marginal advantages in specific tasks like MMLU, it suffers from instability in other domains (e.g., TriviaQA). In contrast, Astro demonstrates the most balanced performance profile across diverse datasets. Crucially, Astro achieves these competitive results while significantly improving quantization efficiency, offering a superior overall trade-off.
Additional results for LLaMA-3.1 are provided in Appendix \ref{app:llama3_results}.

\subsection{Ablation Studies}
\label{subsec:ablation}

We conduct ablation studies on LLaMA-2-7B (W3A16g128) to validate the effectiveness of the proposed regularization components and verify our theoretical assumptions.

\begin{table}[htbp]
\centering
\vspace{-2pt}
\caption{WikiText2 PPL ($\downarrow$) of LLaMA-2-7B (W3A16g128) under varying regularization strengths $\beta$ and different strategies of regularization coefficient $\alpha_k$.}
\vspace{-5pt}
\label{tab:lambda_sensitivity}
\small
\begin{tabularx}{\columnwidth}{lcccccc}
\toprule
\textbf{$\beta$}& \textbf{5e-6} &\textbf{1e-5} & \textbf{3e-5} & \textbf{5e-5} & \textbf{1e-4} & \textbf{5e-4} \\
\midrule
$\alpha_k \propto \|\mathbf{X}^{(k)}\|_F$ & 5.99 & 6.02 & \textbf{5.95} & \textbf{5.95} & 6.05 & 34.91 \\
$\alpha_k = 1$ & 6.00 & 6.03 & 6.02 & 6.01 & 6.04 & 36.41 \\
\bottomrule
\end{tabularx}
\end{table}

\textbf{Impact of Hyperparameter $\beta$ and Activation Guidance.} 
The hyperparameter $\beta$ in Eq.~\eqref{eq:astro_objective_vec} governs the trade-off between reconstruction fidelity and outlier suppression. As shown in Table \ref{tab:lambda_sensitivity}, an insufficiently small $\beta$ fails to effectively curb outliers, resulting in suboptimal performance, whereas an excessive $\beta$ violates the flat minima constraint, leading to accuracy collapse. Furthermore, we investigate the necessity of the activation-guided coefficient $\alpha_k$ (Eq.~\ref{eq:reg_definition}) by comparing it against a uniform regularization ($\alpha_k = 1$) under varying regularization strengths $\beta$, highlighting the necessity of activation-guided regularization.

\begin{table}[htbp]
\centering
\vspace{-4pt}
\caption{Comparison of WikiText2 PPL ($\downarrow$) between the original pre-trained weights and Astro-reconstructed weights among the LLaMA-2 family.}
\vspace{-3pt}
\label{tab:reconstruction_fidelity}
\small
\begin{tabularx}{\columnwidth}{lccc}
\toprule
\textbf{Size} & \textbf{Original PPL} & \textbf{Reconstructed PPL} & \textbf{$\Delta$ PPL} \\
 & (FP16) & (FP16, Unquantized) & \\
\midrule
7B  & 5.47 & 5.48 & +0.01 \\
13B & 4.88 & 4.88 & +0.00 \\
70B & 3.31 & 3.33 & +0.02 \\
\bottomrule
\end{tabularx}
\vspace{-5pt}
\end{table}

\textbf{Validation of the Flat Minima Assumption.}
Astro assumes the existence of an $\epsilon$-equivalent solution manifold where weights can be adjusted without accuracy loss (Theorem \ref{thm:flat_minima}). 
To validate this, we evaluate the perplexity of the \textit{unquantized} reconstructed weights.
As shown in Table \ref{tab:reconstruction_fidelity}, the performance gap compared to the original FP16 weights is negligible across all scales.
This confirms that Astro strictly optimizes quantization robustness within the flat basin of the original loss landscape.

\begin{table}[htbp]
\centering
\caption{Quantitative comparison of Astro-reconstructed weight changes (LLaMA-2-7B, W3A16g128) in three specific groups.}
\vspace{-3pt}
\label{tab:mechanism_analysis}
\small
\setlength{\tabcolsep}{1.3pt}
\begin{tabularx}{\columnwidth}{lccccc}
\toprule
\multirow{2}{*}{\textbf{Group Type}} & \multirow{2}{*}{\textbf{$\|\mathbf{X}^{(k)}\|_F$}} & \multirow{2}{*}{$\alpha_k$} & \multicolumn{3}{c}{\textbf{Max Weight Magnitude} ($\|\mathbf{w}\|_\infty$)} \\
\cmidrule(lr){4-6}
 &  &  & \textbf{Original} & \textbf{Astro} & \textbf{Reduction} \\
\midrule
\textbf{Highest Act.} & 1.886 & 3.305 & 0.100 & 0.062 &  -38.1\% \\
\textbf{Median Act.}  & 0.412 &0.721 & 0.118  & 0.104 & -12.2\% \\
\textbf{Lowest Act.}  & 0.142  &0.248 &  0.109 &0.104 & -4.6 \% \\
\bottomrule
\end{tabularx}
\end{table}

\textbf{Microscopic Mechanism Analysis.}
To examine the treatment of the outlier coupling effect (Theorem \ref{thm:bound}), we evaluate weight magnitude changes across groups with different activation levels.
As shown in Table \ref{tab:mechanism_analysis}, Astro imposes aggressive regularization on the highest-activation groups while applying minimal constraints to low-activation ones, thereby effectively minimizing the theoretical quantization error bound.

\section{Conclusion}
In this paper, we presented \textbf{Astro}, a novel post-training quantization framework that effectively mitigates the accuracy degradation caused by outliers in LLMs. Departing from existing equivalent transformation paradigms, Astro establishes an active weight reconstruction approach that incurs zero additional inference overhead while maintaining seamless compatibility with standard hardware kernels. By leveraging the intrinsic degrees of freedom within the \textit{Flat Minima} of over-parameterized LLMs, we theoretically demonstrate that it is possible to identify a quantization-robust weight configuration strictly equivalent to the original pre-trained weights. Guided by our derived activation-weight coupled error bound, Astro employs an activation-guided structured regularization strategy. This mechanism effectively suppresses the critical weight outliers corresponding to high-activation without model performance degradation. Extensive experiments confirm that Astro achieves competitive performance against mainstream baselines and superior balance between quantization accuracy and efficiency.

\section*{Impact Statement}

This paper presents work whose goal is to advance the field of machine learning. There are many potential societal consequences of our work, none of which we feel must be specifically highlighted here.

\bibliography{example_paper}

@article{lin2024awq,
  title={Awq: Activation-aware weight quantization for on-device llm compression and acceleration},
  author={Lin, Ji and Tang, Jiaming and Tang, Haotian and Yang, Shang and Chen, Wei-Ming and Wang, Wei-Chen and Xiao, Guangxuan and Dang, Xingyu and Gan, Chuang and Han, Song},
  journal={Proceedings of machine learning and systems},
  volume={6},
  pages={87--100},
  year={2024}
}

@inproceedings{
frantar2023optq,
title={{OPTQ}: Accurate Quantization for Generative Pre-trained Transformers},
author={Elias Frantar and Saleh Ashkboos and Torsten Hoefler and Dan Alistarh},
booktitle={The Eleventh International Conference on Learning Representations },
year={2023},
url={https://openreview.net/forum?id=tcbBPnfwxS}
}

@inproceedings{
huang2024billm,
title={Bi{LLM}: Pushing the Limit of Post-Training Quantization for {LLM}s},
author={Wei Huang and Yangdong Liu and Haotong Qin and Ying Li and Shiming Zhang and Xianglong Liu and Michele Magno and XIAOJUAN QI},
booktitle={Forty-first International Conference on Machine Learning},
year={2024},
url={https://openreview.net/forum?id=qOl2WWOqFg}
}

@inproceedings{nipsllmint8,
 author = {Dettmers, Tim and Lewis, Mike and Belkada, Younes and Zettlemoyer, Luke},
 booktitle = {Advances in Neural Information Processing Systems},
 editor = {S. Koyejo and S. Mohamed and A. Agarwal and D. Belgrave and K. Cho and A. Oh},
 pages = {30318--30332},
 publisher = {Curran Associates, Inc.},
 title = {GPT3.int8(): 8-bit Matrix Multiplication for Transformers at Scale},
 volume = {35},
 year = {2022}
}

@inproceedings{xiao2023smoothquant,
  title={Smoothquant: Accurate and efficient post-training quantization for large language models},
  author={Xiao, Guangxuan and Lin, Ji and Seznec, Mickael and Wu, Hao and Demouth, Julien and Han, Song},
  booktitle={International conference on machine learning},
  pages={38087--38099},
  year={2023},
  organization={PMLR}
}

@inproceedings{
kim2024squeezellm,
title={Squeeze{LLM}: Dense-and-Sparse Quantization},
author={Sehoon Kim and Coleman Richard Charles Hooper and Amir Gholami and Zhen Dong and Xiuyu Li and Sheng Shen and Michael W. Mahoney and Kurt Keutzer},
booktitle={Forty-first International Conference on Machine Learning},
year={2024},
url={https://openreview.net/forum?id=0jpbpFia8m}
}

@inproceedings{lee2024owq,
  title={Owq: Outlier-aware weight quantization for efficient fine-tuning and inference of large language models},
  author={Lee, Changhun and Jin, Jungyu and Kim, Taesu and Kim, Hyungjun and Park, Eunhyeok},
  booktitle={Proceedings of the AAAI Conference on Artificial Intelligence},
  volume={38},
  number={12},
  pages={13355--13364},
  year={2024}
}

@article{zhao2024atom,
  title={Atom: Low-bit quantization for efficient and accurate llm serving},
  author={Zhao, Yilong and Lin, Chien-Yu and Zhu, Kan and Ye, Zihao and Chen, Lequn and Zheng, Size and Ceze, Luis and Krishnamurthy, Arvind and Chen, Tianqi and Kasikci, Baris},
  journal={Proceedings of Machine Learning and Systems},
  volume={6},
  pages={196--209},
  year={2024}
}

@inproceedings{
shao2024omniquant,
title={OmniQuant: Omnidirectionally Calibrated Quantization for Large Language Models},
author={Wenqi Shao and Mengzhao Chen and Zhaoyang Zhang and Peng Xu and Lirui Zhao and Zhiqian Li and Kaipeng Zhang and Peng Gao and Yu Qiao and Ping Luo},
booktitle={The Twelfth International Conference on Learning Representations},
year={2024},
url={https://openreview.net/forum?id=8Wuvhh0LYW}
}

@inproceedings{
liu2025spinquant,
title={SpinQuant: {LLM} Quantization with Learned Rotations},
author={Zechun Liu and Changsheng Zhao and Igor Fedorov and Bilge Soran and Dhruv Choudhary and Raghuraman Krishnamoorthi and Vikas Chandra and Yuandong Tian and Tijmen Blankevoort},
booktitle={The Thirteenth International Conference on Learning Representations},
year={2025},
url={https://openreview.net/forum?id=ogO6DGE6FZ}
}

@article{ashkboos2024quarot,
  title={Quarot: Outlier-free 4-bit inference in rotated llms},
  author={Ashkboos, Saleh and Mohtashami, Amirkeivan and Croci, Maximilian L and Li, Bo and Cameron, Pashmina and Jaggi, Martin and Alistarh, Dan and Hoefler, Torsten and Hensman, James},
  journal={Advances in Neural Information Processing Systems},
  volume={37},
  pages={100213--100240},
  year={2024}
}

@article{garipov2018loss,
  title={Loss surfaces, mode connectivity, and fast ensembling of dnns},
  author={Garipov, Timur and Izmailov, Pavel and Podoprikhin, Dmitrii and Vetrov, Dmitry P and Wilson, Andrew G},
  journal={Advances in neural information processing systems},
  volume={31},
  year={2018}
}

@inproceedings{draxler2018essentially,
  title={Essentially no barriers in neural network energy landscape},
  author={Draxler, Felix and Veschgini, Kambis and Salmhofer, Manfred and Hamprecht, Fred},
  booktitle={International conference on machine learning},
  pages={1309--1318},
  year={2018},
  organization={PMLR}
}

@article{kundu2023r2loss,
  title={R2 loss: Range restriction loss for model compression and quantization},
  author={Kundu, Arnav and Yoo, Chungkuk and Mishra, Srijan and Cho, Minsik and Adya, Saurabh},
  journal={arXiv preprint arXiv:2303.08253},
  year={2023}
}

@article{maly2023simple,
  title={A simple approach for quantizing neural networks},
  author={Maly, Johannes and Saab, Rayan},
  journal={Applied and Computational Harmonic Analysis},
  volume={66},
  pages={138--150},
  year={2023},
  publisher={Elsevier}
}

@inproceedings{
li2021brecq,
title={{\{}BRECQ{\}}: Pushing the Limit of Post-Training Quantization by Block Reconstruction},
author={Yuhang Li and Ruihao Gong and Xu Tan and Yang Yang and Peng Hu and Qi Zhang and Fengwei Yu and Wei Wang and Shi Gu},
booktitle={International Conference on Learning Representations},
year={2021},
url={https://openreview.net/forum?id=POWv6hDd9XH}
}

@inproceedings{
dettmers2024spqr,
title={Sp{QR}: A Sparse-Quantized Representation for Near-Lossless {LLM} Weight Compression},
author={Tim Dettmers and Ruslan A. Svirschevski and Vage Egiazarian and Denis Kuznedelev and Elias Frantar and Saleh Ashkboos and Alexander Borzunov and Torsten Hoefler and Dan Alistarh},
booktitle={The Twelfth International Conference on Learning Representations},
year={2024},
url={https://openreview.net/forum?id=Q1u25ahSuy}
}

@article{chee2023quip,
  title={Quip: 2-bit quantization of large language models with guarantees},
  author={Chee, Jerry and Cai, Yaohui and Kuleshov, Volodymyr and De Sa, Christopher M},
  journal={Advances in Neural Information Processing Systems},
  volume={36},
  pages={4396--4429},
  year={2023}
}

@article{zhang2024magr,
  title={Magr: Weight magnitude reduction for enhancing post-training quantization},
  author={Zhang, Aozhong and Wang, Naigang and Deng, Yanxia and Li, Xin and Yang, Zi and Yin, Penghang},
  journal={Advances in neural information processing systems},
  volume={37},
  pages={85109--85130},
  year={2024}
}

@inproceedings{nagel2020adaround,
  title={Up or down? adaptive rounding for post-training quantization},
  author={Nagel, Markus and Amjad, Rana Ali and Van Baalen, Mart and Louizos, Christos and Blankevoort, Tijmen},
  booktitle={International conference on machine learning},
  pages={7197--7206},
  year={2020},
  organization={PMLR}
}

@article{parikh2014proximal,
  title={Proximal algorithms},
  author={Parikh, Neal and Boyd, Stephen and others},
  journal={Foundations and trends{\textregistered} in Optimization},
  volume={1},
  number={3},
  pages={127--239},
  year={2014},
  publisher={Now Publishers, Inc.}
}

@article{touvron2023llama2,
  title={Llama 2: Open foundation and fine-tuned chat models},
  author={Touvron, Hugo and Martin, Louis and Stone, Kevin and Albert, Peter and Almahairi, Amjad and Babaei, Yasmine and Bashlykov, Nikolay and Batra, Soumya and Bhargava, Prajjwal and Bhosale, Shruti and others},
  journal={arXiv preprint arXiv:2307.09288},
  year={2023}
}

@article{grattafiori2024llama3,
  title={The llama 3 herd of models},
  author={Grattafiori, Aaron and Dubey, Abhimanyu and Jauhri, Abhinav and Pandey, Abhinav and Kadian, Abhishek and Al-Dahle, Ahmad and Letman, Aiesha and Mathur, Akhil and Schelten, Alan and Vaughan, Alex and others},
  journal={arXiv preprint arXiv:2407.21783},
  year={2024}
}

@misc{eval-harness,
  author       = {Gao, Leo and Tow, Jonathan and Abbasi, Baber and Biderman, Stella and Black, Sid and DiPofi, Anthony and Foster, Charles and Golding, Laurence and Hsu, Jeffrey and Le Noac'h, Alain and Li, Haonan and McDonell, Kyle and Muennighoff, Niklas and Ociepa, Chris and Phang, Jason and Reynolds, Laria and Schoelkopf, Hailey and Skowron, Aviya and Sutawika, Lintang and Tang, Eric and Thite, Anish and Wang, Ben and Wang, Kevin and Zou, Andy},
  title        = {The Language Model Evaluation Harness},
  month        = 07,
  year         = 2024,
  publisher    = {Zenodo},
  version      = {v0.4.3},
  doi          = {10.5281/zenodo.12608602},
  url          = {https://zenodo.org/records/12608602}
}

@inproceedings{
merity2016pointer-wiki,
title={Pointer Sentinel Mixture Models},
author={Stephen Merity and Caiming Xiong and James Bradbury and Richard Socher},
booktitle={International Conference on Learning Representations},
year={2017},
url={https://openreview.net/forum?id=Byj72udxe}
}

@inproceedings{
hendrycks2021measuring-mmlu,
title={Measuring Massive Multitask Language Understanding},
author={Dan Hendrycks and Collin Burns and Steven Basart and Andy Zou and Mantas Mazeika and Dawn Song and Jacob Steinhardt},
booktitle={International Conference on Learning Representations},
year={2021},
url={https://openreview.net/forum?id=d7KBjmI3GmQ}
}

@article{clark2018think-arc,
  title={Think you have solved question answering? try arc, the ai2 reasoning challenge},
  author={Clark, Peter and Cowhey, Isaac and Etzioni, Oren and Khot, Tushar and Sabharwal, Ashish and Schoenick, Carissa and Tafjord, Oyvind},
  journal={arXiv preprint arXiv:1803.05457},
  year={2018}
}

@article{zellers2019hellaswag,
  title={Hellaswag: Can a machine really finish your sentence?},
  author={Zellers, Rowan and Holtzman, Ari and Bisk, Yonatan and Farhadi, Ali and Choi, Yejin},
  journal={arXiv preprint arXiv:1905.07830},
  year={2019}
}

@article{kwiatkowski2019natural,
  title={Natural questions: a benchmark for question answering research},
  author={Kwiatkowski, Tom and Palomaki, Jennimaria and Redfield, Olivia and Collins, Michael and Parikh, Ankur and Alberti, Chris and Epstein, Danielle and Polosukhin, Illia and Devlin, Jacob and Lee, Kenton and others},
  journal={Transactions of the Association for Computational Linguistics},
  volume={7},
  pages={453--466},
  year={2019},
  publisher={MIT Press One Rogers Street, Cambridge, MA 02142-1209, USA journals-info~…}
}

@article{joshi2017triviaqa,
  title={Triviaqa: A large scale distantly supervised challenge dataset for reading comprehension},
  author={Joshi, Mandar and Choi, Eunsol and Weld, Daniel S and Zettlemoyer, Luke},
  journal={arXiv preprint arXiv:1705.03551},
  year={2017}
}

@article{frantar2022obq,
  title={Optimal brain compression: A framework for accurate post-training quantization and pruning},
  author={Frantar, Elias and Alistarh, Dan},
  journal={Advances in Neural Information Processing Systems},
  volume={35},
  pages={4475--4488},
  year={2022}
}
\bibliographystyle{icml2026}

\newpage
\appendix
\onecolumn

\section{Proofs}

\subsection{Proof of Theorem \ref{thm:flat_minima}}
\label{app:proof_flat_minima}
\begin{proof}
We analyze the loss landscape around the converged parameters $\bm{\Theta}_{\text{orig}}$ using a second-order Taylor expansion. Let $\hat{\calL}(\bm{\Theta})$ denote the quadratic approximation of the loss function. For a parameter perturbation $\Delta \bm{\Theta}$, the change in loss is given by:
\begin{equation}
    \Delta \hat{\calL} = \hat{\calL}(\bm{\Theta}_{\text{orig}} + \Delta \bm{\Theta}) - \calL(\bm{\Theta}_{\text{orig}}) = \nabla \calL(\bm{\Theta}_{\text{orig}})^\top \Delta \bm{\Theta} + \frac{1}{2} \Delta \bm{\Theta}^\top \mathbf{H} \Delta \bm{\Theta}.
\end{equation}
Since $\bm{\Theta}_{\text{orig}}$ corresponds to a local minimum (or a converged stationary point), the gradient vanishes ($\nabla \calL(\bm{\Theta}_{\text{orig}}) \approx \mathbf{0}$). The loss degradation is thus dominated by the quadratic form:
\begin{equation}
    \Delta \hat{\calL} = \frac{1}{2} \Delta \bm{\Theta}^\top \mathbf{H} \Delta \bm{\Theta}.
\end{equation}
Let $\mathbf{H} = \mathbf{Q} \mathbf{\Lambda} \mathbf{Q}^\top$ be the eigendecomposition of the Hessian, where $\mathbf{Q} = [\mathbf{v}_1, \dots, \mathbf{v}_D]$ are the orthonormal eigenvectors and $\mathbf{\Lambda} = \text{diag}(\lambda_1, \dots, \lambda_D)$ are the eigenvalues. 

We restrict our search to perturbations $\Delta \bm{\Theta}$ lying strictly within the \textit{Flat Subspace} $\mathcal{V}_{\text{flat}}$ (Definition \ref{def:flat_subspace}). Consequently, $\Delta \bm{\Theta}$ can be expressed as a linear combination of the eigenvectors associated with the tail spectrum indices $j \in \{r+1, \dots, D\}$:
\begin{equation}
    \Delta \bm{\Theta} = \sum_{j=r+1}^D c_j \mathbf{v}_j, \quad \text{where } c_j = \mathbf{v}_j^\top \Delta \bm{\Theta}.
\end{equation}
Substituting this into the quadratic form, and utilizing the orthonormality of eigenvectors:
\begin{equation}
    \Delta \hat{\calL} = \frac{1}{2} \left(\sum_{i=r+1}^D c_i \mathbf{v}_i\right)^\top \mathbf{H} \left(\sum_{j=r+1}^D c_j \mathbf{v}_j\right) = \frac{1}{2} \sum_{j=r+1}^D \lambda_j c_j^2.
\end{equation}
To derive a bound on the loss degradation, we apply the absolute value inequality and invoke Assumption~\ref{ass:hessian}, which states that $|\lambda_j| \le \gamma$ for all $j > r$:
\begin{equation}
    |\Delta \hat{\calL}| \le \frac{1}{2} \sum_{j=r+1}^D |\lambda_j| c_j^2 \le \frac{1}{2} \gamma \sum_{j=r+1}^D c_j^2.
\end{equation}
Recognizing that $\|\Delta \bm{\Theta}\|_2^2 = \sum_{j=r+1}^D c_j^2$ (Parseval's identity), we arrive at the upper bound:
\begin{equation}
    |\Delta \hat{\calL}| \le \frac{1}{2} \gamma \|\Delta \bm{\Theta}\|_2^2.
\end{equation}
To ensure the loss degradation remains within the tolerance $\epsilon$ (i.e., $|\Delta \hat{\calL}| \le \epsilon$), it is \textbf{sufficient} to enforce:
\begin{equation}
    \frac{1}{2} \gamma \|\Delta \bm{\Theta}\|_2^2 \le \epsilon \implies \|\Delta \bm{\Theta}\|_2 \le \sqrt{\frac{2\epsilon}{\gamma}}.
\end{equation}
Let $\delta = \sqrt{2\epsilon/\gamma}$. We define the region $\mathcal{M}_{\epsilon}$ as the intersection of the Flat Subspace and the Euclidean ball of radius $\delta$:
\begin{equation}
    \mathcal{M}_{\epsilon} = \{ \bm{\Theta} = \bm{\Theta}_{\text{orig}} + \Delta \bm{\Theta} \mid \Delta \bm{\Theta} \in \mathcal{V}_{\text{flat}}, \|\Delta \bm{\Theta}\|_2 \le \delta \}.
\end{equation}
Geometric Interpretation: While the exact isosurface of the quadratic loss is a degenerate ellipsoid stretched infinitely along flat directions, $\mathcal{M}_{\epsilon}$ represents a conservative spherical ``safe zone'' within this ellipsoid based on the worst-case curvature $\gamma$. For any $\bm{\Theta} \in \mathcal{M}_{\epsilon}$, the condition $|\calL(\bm{\Theta}) - \calL(\bm{\Theta}_{\text{orig}})| \le \epsilon$ is guaranteed under the quadratic model.
\end{proof}

\subsection{Proof of Theorem \ref{thm:bound} (Error Bound)}
\label{app:proof_bound}

\begin{proof}
Let $\Delta \bW = \bW - \mathcal{Q}(\bW)$ denote the quantization error matrix. We aim to bound the reconstruction error $\calE = \|\bX \Delta \bW\|_F$.

\textbf{1. Granularity of Quantization Noise.}
Consider a specific weight group vector $\mathbf{w} \in \mathbb{R}^g$ (corresponding to the $k$-th input group and $j$-th output channel). In symmetric uniform quantization with bit-width $b$, the scaling factor $s$ is determined by the maximum absolute value (outlier) in the group:
\begin{equation}
    s = \frac{\|\mathbf{w}\|_\infty}{2^{b-1}-1}.
\end{equation}
The quantized value is obtained by rounding to the nearest integer grid. Assuming no clipping truncation (since the scale is defined by the maximum value), the quantization error $\delta w_i$ for each element $w_i \in \mathbf{w}$ is bounded by half the discretization step size:
\begin{equation}
    |\delta w_i| \le \frac{s}{2} = \frac{\|\mathbf{w}\|_\infty}{2(2^{b-1}-1)}.
\end{equation}

\textbf{2. Bound on Weight Error Vector.}
We compute the Euclidean norm ($L_2$) of the quantization error vector $\Delta \mathbf{w} \in \mathbb{R}^g$ for this group:
\begin{equation}
    \|\Delta \mathbf{w}\|_2 = \sqrt{\sum_{i=1}^g (\delta w_i)^2} \le \sqrt{\sum_{i=1}^g \left( \frac{s}{2} \right)^2} = \frac{s\sqrt{g}}{2}.
\end{equation}
Substituting $s$, we obtain the bound linked to the weight outlier:
\begin{equation}
    \|\Delta \mathbf{w}\|_2 \le \frac{\sqrt{g}}{2(2^{b-1}-1)} \|\mathbf{w}\|_\infty.
\end{equation}

\textbf{3. Layer-wise Error Aggregation.}
The total error is the Frobenius norm of the matrix product $\bX \Delta \bW$. Let $\Delta \bW_{:,j}$ denote the $j$-th column of the noise matrix. The squared Frobenius norm decomposes column-wise:
\begin{equation}
    \calE^2 = \|\bX \Delta \bW\|_F^2 = \sum_{j=1}^{C_{out}} \|\bX \Delta \bW_{:,j}\|_2^2.
\end{equation}
For a single output channel $j$, the matrix-vector product can be decomposed into the sum of $K$ groups. Let $\bX^{(k)}$ be the sub-matrix of activations and $\Delta \mathbf{w}^{(k,j)}$ be the noise vector for the $k$-th group. Using the triangle inequality:
\begin{equation}
    \|\bX \Delta \bW_{:,j}\|_2 = \left\| \sum_{k=1}^K \bX^{(k)} \Delta \mathbf{w}^{(k,j)} \right\|_2 \le \sum_{k=1}^K \left\| \bX^{(k)} \Delta \mathbf{w}^{(k,j)} \right\|_2.
\end{equation}
Using the consistency property of matrix norms ($\|\mathbf{A}\mathbf{x}\|_2 \le \|\mathbf{A}\|_2 \|\mathbf{x}\|_2$) and the fact that the spectral norm is bounded by the Frobenius norm ($\|\mathbf{A}\|_2 \le \|\mathbf{A}\|_F$), we have:
\begin{equation}
    \left\| \bX^{(k)} \Delta \mathbf{w}^{(k,j)} \right\|_2 \le \|\bX^{(k)}\|_F \|\Delta \mathbf{w}^{(k,j)}\|_2.
\end{equation}
Combining this with the result from Step 2:
\begin{equation}
    \|\bX \Delta \bW_{:,j}\|_2 \le \sum_{k=1}^K \|\bX^{(k)}\|_F \cdot \left( \frac{\sqrt{g}}{2(2^{b-1}-1)} \|\bW^{(k,j)}\|_\infty \right).
\end{equation}

\textbf{4. Global Bound.}
Finally, we bound the total Frobenius norm $\calE$. Using the property that for a vector $\mathbf{v} \in \mathbb{R}^{C_{out}}$, $\|\mathbf{v}\|_2 \le \|\mathbf{v}\|_1$, we can sum the error bounds across output channels (conceptually treating the error bound of each channel as an element in a vector):
\begin{equation}
    \calE = \sqrt{\sum_{j=1}^{C_{out}} \|\bX \Delta \bW_{:,j}\|_2^2} \le \sum_{j=1}^{C_{out}} \|\bX \Delta \bW_{:,j}\|_2.
\end{equation}
Substituting the per-channel bound:
\begin{equation}
    \calE \le \frac{\sqrt{g}}{2(2^{b-1}-1)} \sum_{j=1}^{C_{out}} \sum_{k=1}^{K} \|\bX^{(k)}\|_F \|\bW^{(k,j)}\|_\infty.
\end{equation}
This completes the proof.
\end{proof}

\section{Additional Experimental Results on LLaMA-3}
\label{app:llama3_results}

In this section, we provide additional experimental results on the LLaMA-3.1-8B \cite{grattafiori2024llama3}. 
We evaluate the performance using the W3A16g128 quantization configuration.
Baselines include OmniQuant \cite{shao2024omniquant} and MagR \cite{zhang2024magr}.
Table \ref{tab:llama3_full} details accuracy on various benchmarks (MMLU, ARC, HellaSwag, etc.).

Consistent with the observations on LLaMA-2, Astro achieves the best competitive accuracy, further validating that the activation-guided structured regularization effectively handles the evolved outliers.

\begin{table*}[h]
\centering
\caption{Comparison of W3A16g128 LLaMA-3.1-8B PTQ performance on zero-shot/few-shot benchmarks. \textbf{Bold} indicates the best result.}
\label{tab:llama3_full}
\resizebox{\textwidth}{!}{
\begin{tabular}{lcccccccccc}
\toprule
\multirow{2}{*}{\textbf{Method}} & \multicolumn{5}{c}{\textbf{MMLU (5-shot)}} & \multicolumn{2}{c}{\textbf{Knowledge (5-shot)}} & \multicolumn{3}{c}{\textbf{Common Sense (0-shot)}} \\
\cmidrule(lr){2-6} \cmidrule(lr){7-8} \cmidrule(lr){9-11}
 & \textbf{Avg.} & \textit{STEM} & \textit{Human.} & \textit{Social} & \textit{Other} & \textbf{TriviaQA} & \textbf{NQ} & \textbf{ARC-C} & \textbf{ARC-E} & \textbf{HellaS.} \\
\midrule
FP16 & 65.3 & 56.0 & 59.8 & 76.4 & 72.2 & 70.4 & 28.5 & 55.2 & 82.4 & 79.3 \\
\midrule
OmniQuant & 51.7 & 45.5 & 45.9 & 59.3 & 59.1 & 47.6 & 14.8 & 45.1 & 71.6 & 72.3 \\
MagR & 51.9 & 46.1 & 46.1 & 59.7 & 58.9 & 45.5 & 15.8 & 44.7 & 73.6 & 72.0 \\
\textbf{Astro (Ours)} & \textbf{56.3} & \textbf{49.7} & \textbf{50.4} & \textbf{65.6} & \textbf{62.6} & \textbf{51.2} & \textbf{16.8} & \textbf{48.0} & \textbf{77.1} & \textbf{74.0} \\
\bottomrule
\end{tabular}
}
\end{table*}

\section{Implementation Algorithms}
\label{app:algorithms}

We provide the detailed pseudocode for the efficient $L_1$-ball projection used in the proximal step.

\begin{algorithm}[h]
\caption{Efficient Projection onto the $L_1$-Ball}
\label{alg:l1_proj}
\begin{algorithmic}[1]
\STATE \textbf{Input:} Vector $\mathbf{v} \in \mathbb{R}^g$, Radius $r=1$
\STATE \textbf{Output:} Projected vector $\mathbf{w} \in \mathbb{R}^g$ such that $\|\mathbf{w}\|_1 \le r$
\STATE
\STATE \textbf{Step 1: Check trivial case}
\IF{$\|\mathbf{v}\|_1 \le r$}
    \STATE \textbf{return} $\mathbf{v}$
\ENDIF
\STATE
\STATE \textbf{Step 2: Sorting-based Projection}
\STATE $\mathbf{u} \leftarrow \text{sort}(|\mathbf{v}|)$ in descending order
\STATE Compute cumulative sums: $S_j = \sum_{i=1}^j u_i$
\STATE Find the number of active components $\rho$:
\STATE $\rho = \max \{ j \in [1, g] : u_j - \frac{1}{j}(S_j - r) > 0 \}$
\STATE
\STATE \textbf{Step 3: Compute Threshold}
\STATE $\theta = \frac{1}{\rho} (S_\rho - r)$
\STATE
\STATE \textbf{Step 4: Soft-Thresholding}
\STATE $\mathbf{w} \leftarrow \text{sign}(\mathbf{v}) \odot \max(|\mathbf{v}| - \theta, 0)$
\STATE \textbf{return} $\mathbf{w}$
\end{algorithmic}
\end{algorithm}

\section{Detailed Analysis of Inference Overhead in Transformation-based PTQ}
\label{app:inference_overhead}

In previous sections, we identified a critical limitation of existing equivalent transformation methods, such as AWQ, OmniQuant, SpinQuant, and QuaRot. These approaches either incur additional inference latency or necessitate the development of complex, non-standard CUDA kernels. In this appendix, we provide a rigorous analysis of these constraints through the lenses of hardware efficiency and operator fusion difficulty.

\subsection{Memory Bandwidth Constraints in LLM Decoding}

Consider the standard linear layer computation within a quantized LLM. Let $\mathbf{X} \in \mathbb{R}^{N \times C_{in}}$ denote the input activation tensor and $\mathcal{Q}(\mathbf{W}) \in \mathbb{R}^{C_{in} \times C_{out}}$ represent the quantized weights. The standard inference operation is defined as:
\begin{equation}
    \mathbf{Y} = \mathbf{X} \cdot \mathcal{Q}(\mathbf{W}).
\end{equation}
Transformation-based methods introduce an invertible transformation matrix $\mathbf{T} \in \mathbb{R}^{C_{in} \times C_{in}}$ to smooth outliers. Consequently, the inference flow is modified as follows:
\begin{equation}
    \mathbf{Y} = (\mathbf{X}\mathbf{T}^{-1}) \cdot \mathcal{Q}(\mathbf{T}\mathbf{W}).
\end{equation}
While the weight transformation $\tilde{\mathbf{W}} = \mathcal{Q}(\mathbf{T}\mathbf{W})$ is performed offline, the activation transformation $\tilde{\mathbf{X}} = \mathbf{X}\mathbf{T}^{-1}$ need to be executed online during each inference step. 

The arithmetic complexity (FLOPs) of the transformation $\mathbf{X}\mathbf{T}^{-1}$ often appears negligible relative to the primary projection $\mathbf{X}\mathbf{W}$. However, in the context of LLM inference, particularly during the autoregressive decoding phase, performance is strictly memory-bound rather than compute-bound. If the transformation is implemented as a separate kernel prior to the main General Matrix Multiply (GEMM), it introduces a severe I/O penalty. Specifically, the GPU must perform the following sequence of operations:
\begin{enumerate}
    \item Load the raw activation $\mathbf{X}$ from High Bandwidth Memory (HBM) to on-chip SRAM;
    \item Execute the transformation $\tilde{\mathbf{X}} \leftarrow \mathbf{X}\mathbf{T}^{-1}$ within the functional units;
    \item Store the transformed $\tilde{\mathbf{X}}$ back to HBM;
    \item Reload $\tilde{\mathbf{X}}$ from HBM to SRAM for the subsequent GEMM operation.
\end{enumerate}
This "Read-Modify-Write-Read" pattern significantly increases the total memory traffic. For memory-bound decoding tasks, this overhead directly translates into increased latency, thereby negating some of the efficiency gains typically expected from low-bit quantization.

\subsection{Barriers to Offline Fusion}

To mitigate the I/O overhead associated with activation transformations, one might attempt to mathematically fuse the transformation matrix $\mathbf{T}^{-1}$ into preceding layers. However, such a strategy encounters fundamental structural barriers within modern Large Language Model (LLM) architectures.

\paragraph{1. Non-Commutativity with Activation Functions.}
Ideally, $\mathbf{T}^{-1}$ would be absorbed into the weight matrix of the preceding layer, $\mathbf{W}_{\text{prev}}$. However, modern LLMs typically interpose non-linear activation functions $\sigma(\cdot)$ between successive layers, such that the data flow is defined by:
\begin{equation}
    \mathbf{X}_{\text{curr}} = \sigma(\mathbf{X}_{\text{prev}} \mathbf{W}_{\text{prev}}).
\end{equation}
Eliminating the online computation of $\mathbf{T}^{-1}$ requires propagating the transformation backward through the non-linearity. This operation is only feasible if the commutativity property $\sigma(\mathbf{x}) \cdot \mathbf{T}^{-1} \equiv \sigma(\mathbf{x} \cdot \mathbf{T}^{-1})$ holds. While homogeneous functions such as ReLU satisfy this condition for diagonal scaling matrices with non-negative entries, modern architectures predominantly utilize SiLU or GELU. Because these functions are not scale-invariant, it follows that $\text{SiLU}(x) \cdot s \neq \text{SiLU}(x \cdot s)$. Consequently, the transformation cannot be fused into the previous layer including such activation functions, thereby necessitating explicit online execution during inference.

\paragraph{2. Structural Mismatch in Normalization Layers.}
Rotation-based methods, such as SpinQuant and QuaRot, encounter another challenge involving normalization layers. The operator preceding an Attention or MLP block is typically RMSNorm, defined as:
\begin{equation}
    \text{RMSNorm}(\mathbf{x}) = \frac{\mathbf{x}}{\|\mathbf{x}\|_2} \odot \mathbf{g},
\end{equation}
where $\mathbf{g}$ denotes a learnable diagonal scaling vector. Attempting to fuse a dense rotation matrix $\mathbf{R}^\top$ into this layer implies the computation of $\mathbf{g}_{\text{new}}(\mathbf{x}) = (\mathbf{x} \odot \mathbf{g}) \mathbf{R}^\top$, which fundamentally alters the operator's complexity class. Specifically, the standard RMSNorm is an element-wise operation with $\mathcal{O}(C_{in})$ parameters and complexity. In contrast, the fused rotation effectively transforms the operator into a dense linear layer with $\mathcal{O}(C_{in}^2)$ parameters. 

Such a structural modification precludes the use of highly optimized kernels
Although the orthogonality of $\mathbf{R}$ ensures the preservation of the Euclidean norm, allowing $\mathbf{R}^\top$ to theoretically bypass the normalization if $\mathbf{g}$ is absorbed elsewhere, this approach necessitates significant alterations to the canonical model architecture. Such intrusive modifications require re-performing complex operator fusion during compilation, which limits the flexibility and deployability of the resulting model.

\textbf{Conclusion.} In contrast to transformation-based techniques, Astro suppresses outliers by strictly reconstructing weights within the standard solution space. The resulting weights $\tilde{\mathbf{W}}$ are structurally identical to the original parameters, which guarantees \textbf{zero inference overhead}. Furthermore, this approach ensures total compatibility with existing compilation stacks and hardware-specific kernels, providing a decisive advantage for industrial-scale deployment.

\section{Generalizability Across Base Quantization Algorithms}
\label{app:ablition_rtn}

To demonstrate the architectural agnosticism and generalizability of the Astro framework, we evaluate its effectiveness when integrated with Round-to-Nearest (RTN) quantization. Unlike GPTQ, which utilizes Hessian-based information for layer-wise reconstruction, RTN represents the most fundamental quantization paradigm. By applying Astro as a pre-processing regularization step for RTN, we can isolate the impact of our activation-guided structured regularization from the influence of complex optimization-based reconstruction.

As summarized in Table \ref{tab:rtn_generalization}, the experimental results on LLaMA-2-7B (W3A16g128) confirm that Astro provides consistent performance improvements regardless of the underlying quantization algorithm. Specifically, compared to MagR, Astro achieves better perplexity reduction over the vanilla RTN baseline. 
These findings corroborate the benefits of our activation-guided outlier suppression. 

\begin{table}[htbp]
\centering
\vspace{-5pt}
\caption{WikiText-2 PPL ($\downarrow$) of LLaMA-2-7B (W3A16g128) using Round-to-Nearest (RTN) as the base quantization method. The results demonstrate that Astro effectively enhances even the most basic quantization paradigms.}
\label{tab:rtn_generalization}
\small
\begin{tabularx}{0.7\columnwidth}{l@{\extracolsep{\fill}}ccc}
\toprule
\textbf{Method} &  RTN (Vanilla) & MagR + RTN & Astro + RTN\\
\midrule
\textbf{WikiText-2 PPL $\downarrow$} & 6.66 & 6.48 & \textbf{6.40} \\
\bottomrule
\end{tabularx}
\end{table}


\end{document}